\documentclass[10pt,letterpaper]{article}

\usepackage[utf8]{inputenc}

\usepackage{cite}

\usepackage{nameref,hyperref}

\usepackage{microtype}
\DisableLigatures[f]{encoding = *, family = * }

\usepackage{epstopdf}

\usepackage{color}

\usepackage{caption}
\usepackage{subcaption}
\usepackage{amsmath}
\usepackage{amsfonts}
\usepackage{algorithm2e}
\usepackage{makecell}
\usepackage{multirow}

\definecolor{Gray}{gray}{.25}

\usepackage{graphicx}

\usepackage{amsthm}

\newtheorem{lemma}{Lemma}

\usepackage{wrapfig}
\usepackage[pscoord]{eso-pic}

\begin{document}
\vspace*{0.35in}

% title goes here:
\begin{flushleft}
{\Large
\textbf\newline{ClusterGAN : Latent Space Clustering in \\ Generative Adversarial Networks}
}
\newline
% authors go here:
\\
Sudipto Mukherjee\textsuperscript{1},
Himanshu Asnani\textsuperscript{1},
Eugene Lin\textsuperscript{1},
Sreeram Kannan\textsuperscript{1},
%Author 5\textsuperscript{2},
%Author 6\textsuperscript{2},
%Author 7\textsuperscript{1,*}
\\
\bigskip
\bf{1} University of Washington, Seattle.
%\\
%\bf{2} Affiliation B
\\
\bigskip
 \{sudipm, asnani, lines, ksreeram\}@uw.edu

\end{flushleft}

\section*{Abstract}
Generative Adversarial networks (GANs) have obtained remarkable success in many unsupervised learning tasks and unarguably, clustering is an important unsupervised learning problem. While one can potentially exploit the latent-space back-projection in GANs to cluster, we demonstrate that the cluster structure is not retained in the GAN latent space. 
	In this paper, we propose ClusterGAN as a new mechanism for clustering using GANs. By sampling latent variables from a mixture of one-hot encoded variables and continuous latent variables, coupled with an inverse network (which projects the data to the latent space) trained jointly with a clustering specific loss, we are able to achieve clustering in the latent space. Our results show a remarkable phenomenon that GANs can preserve latent space interpolation across categories, even though the discriminator is never exposed to such vectors. We compare our results with various clustering baselines and demonstrate superior performance on both synthetic and real datasets.

% now start line numbers
%\linenumbers

% the * after section prevents numbering
\section{Introduction}

\subsection{Motivation}

\noindent Representation learning enables machine learning models to decipher underlying semantics in data and disentangle hidden factors of variation. These powerful representations have made it possible to transfer knowledge across various tasks. But what makes one representation better than another ? \cite{bengio} mentioned several general-purpose priors that are not dependent on the downstream task, but appear as commonalities in good representations. One of the general-purpose priors of representation learning that is ubiquitous across data intensive domains is clustering. Clustering has been extensively studied in unsupervised learning with multifarious approaches seeking efficient algorithms \cite{ngspec}, problem specific distance metrics \cite{xiang}, validation \cite{halkidi} and the like. Even though the main focus of clustering has been to separate out the original data into classes, it would be even nicer if such clustering was obtained along with dimensionality reduction where the real data actually seems to come from a lower dimensional manifold. 

In recent times, much of unsupervised learning is driven by deep generative approaches, the two most prominent being Variational Autoencoder (VAE) \cite{kingma} and Generative Adversarial Network (GAN) \cite{goodfellow}. The popularity of generative models themselves is hinged upon the ability of these models to capture high dimensional probability distributions, imputation of missing data and dealing with multimodal outputs. Both GAN and VAE aim to match the real data distribution (VAE using an explicit approximation of maximum likelihood and GAN through implicit sampling), and simulataneously provide a mapping from a latent space $\mathcal{Z}$ to the input space $\mathcal{X}$. %This is in contrast to recent autoregressive generative approaches \cite{van} \cite{wavenet} that depart from latent variable models.% 
The latent space of GANs not only provides dimensionality reduction, but also gives rise to novel applications. Perturbations in the latent space could be used to determine adversarial examples that further help build robust classifiers \cite{ilyas}. Compressed sensing using GANs \cite{bora} relies on finding a latent vector that minimizes the reconstruction error for the measurements. Generative compression is yet another application involving $\mathcal{Z}$ \cite{santurkar}. One of the most fascinating outcomes of the GAN training is the interpolation in the latent space. Simple vector arithmetic properties emerge which when manipulated lead to changes in the semantic qualities of the generated images \cite{radford}. This differentiates GANs from traditional dimensionality reduction techniques \cite{mika} \cite{maaten} which lack interpretability. One potential application that demands such a property is clustering of cell types in genomics. GANs provide a means to understand the change in high-dimensional gene expression as one traverses from one cell type (i.e., cluster) to another in the latent space. Here, it is critical to have both clustering as well as good interpretability and interpolation ability. This brings us to the principal motivation of this work: \textbf{\textit{Can we design a GAN training methodology that clusters in the latent space?}}  %But GANs can produce realistic samples which are of better visual quality than VAE.

\subsection{Related Works}

Deep learning approaches have been used for dimensionality reduction starting with variants of the autoencoder such as the stacked denoising autoencoders \cite{vincent}, sparse autoencoder \cite{coates} and deep CCA \cite{andrew}. Architectures for deep unsupervised subspace clustering have also been built on the  encoder-decoder framework \cite{ji}. Recent works have addressed this problem of joint clustering and dimensionality reduction in autoencoders. \cite{xie} solved this problem by initializing the cluster centroids and the embedding with a stacked autoencoder. Then they use alternating optimization to improve the clustering and report state-of-the-art results in both clustering accuracy and speed on real datasets. The clustering algorithm is referred to as DEC in their paper.  Since K-means is often the most widely used algorithm for clustering, \cite{yang} improved upon DEC by introducing a modified cost function that incorporates the K-means loss. They optimized the non-convex objective using alternating SGD to obtain an embedding that is amenable to K-means clustering. Their algorithm DCN was shown to outperform all standard clustering methods on a range of datasets. It is interesting to note that the vanilla autoencoder by itself did not explicitly have any clustering objective. But it could be improved to achieve this end by careful algorithmic design. Since GANs have outperformed autoencoders in generating high fidelty samples, we had a strong intuition in favour of the powerful latent representations of GAN providing improved clustering performance also. 

\begin{figure}
\includegraphics[width=1.00\textwidth]{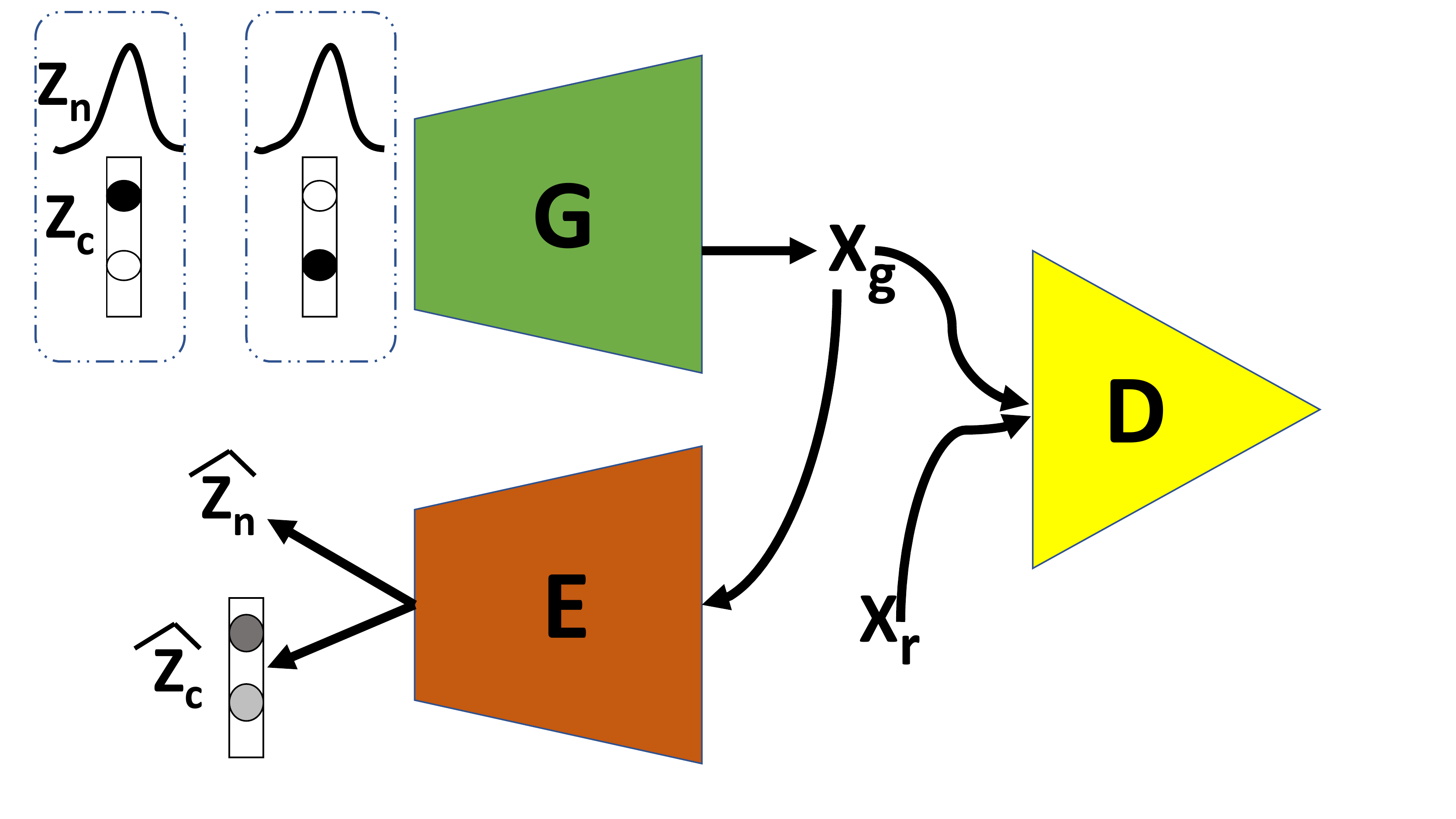} 
\caption{ClusterGAN Architecture}
\label{arch}
\end{figure}

Interpretable representation learning in the latent space has been investigated for GANs in the seminal work of \cite{infogan}. The authors trained a GAN with an additional term in the loss that seeks to maximize the mutual information between a subset of the generator's noise variables and the generated output. The key goal of InfoGAN is to create interpretable and disentangled latent variables. While InfoGAN does employ discrete latent variables, it is not specifically designed for clustering. In this paper, we show that our proposed architecture is superior to InfoGAN for clustering. The other prominent family of generative models, VAE, has the additional advantage of having an inference network, the encoder, which is jointly learnt during training. This enables mapping from $\mathcal{X}$ to $\mathcal{Z}$ that could potentially preserve cluster structure by suitable algorithmic design. Unfortunately, no such inference mechanism exists in GANs, let alone the possibility of clustering in the latent space.To bridge the gap between VAE and GAN, various methods such as Adversarially Learned Inference (ALI) \cite{dumoulin}, Bidirectional Generative Adversarial Networks (BiGAN) \cite{donahue} have introduced an inference network which is trained to match the joint distributions  of $(x, z)$ learnt by the encoder $\mathcal{E}$ and decoder $\mathcal{G}$ networks. Typically, the reconstruction in ALI/BiGAN is poor as there is no deterministic pointwise matching between $x$ and $\mathcal{G}(\mathcal{E}(x))$ involved in the training. Architectures such as Wasserstein Autoencoder \cite{tolstikhin}, Adversarial Autoencoder \cite{makhzani},  which depart from the traditional GAN framework,  also have an encoder as part of the network. So this led us to consider a formulation using an Encoder which could \textit{both reduce the cycle loss as well as aid in clustering.} 

\subsection{Main Contributions}
To the best of our knowledge, this is the first work that addresses the problem of clustering in the latent space of GAN. The main contributions of the paper can be summarized as follows:

\begin{itemize}
\item[$\bullet$] We show that even though the GAN latent variable preserves information about the observed data, the latent points are smoothly scattered based on the latent distribution leading to no observable clusters. 
\item[$\bullet$] We propose three main algorithmic ideas in ClusterGAN in order to remedy this situation. 

\begin{enumerate}
\item We utilize a \textbf{mixture of discrete and continuous} latent variables in order to create a non-smooth geometry in the latent space. 
\item We propose a \textbf{novel backpropogation algorithm} accommodating the discrete-continuous mixture, as well as an \textbf{explicit inverse-mapping network} to obtain the latent variables given the data points, since the problem is non-convex. 
\item We propose to jointly train the GAN along with the inverse-mapping network with \textbf{a clustering-specific loss} so that the distance geometry in the projected space reflects the distance-geometry of the variables. 
\end{enumerate}

\item[$\bullet$] We compare ClusterGAN and other possible GAN based clustering algorithms, such as InfoGAN, along with multiple clustering baselines on varied datasets. This demonstrates the superior performance of ClusterGAN  for the clustering task. 
\item[$\bullet$] We demonstrate that ClusterGAN surprisingly retains good interpolation across the different classes (encoded using one-hot latent variables), even though the discriminator is never exposed to such samples.
\end{itemize}
The formulation is general enough to provide a \textit{meta} framework that incorporates the additional desirable property of clustering in GAN training.  

\begin{figure}
\begin{subfigure}[b]{0.50\textwidth}
\includegraphics[width=\textwidth]{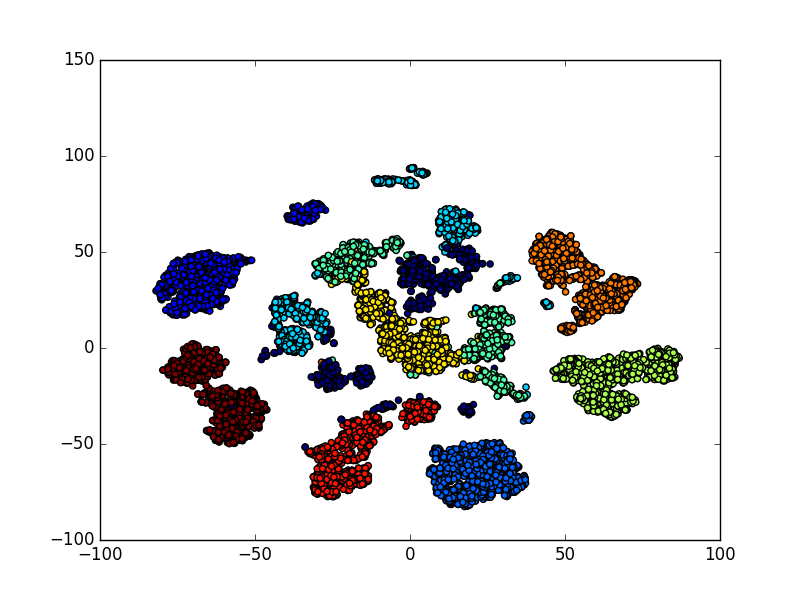} 
\caption{Non-linear generator with $z \sim \mathcal{N}(0, I)$}
\label{non-line}
\end{subfigure}
\begin{subfigure}[b]{0.50\textwidth}
\includegraphics[width=\textwidth]{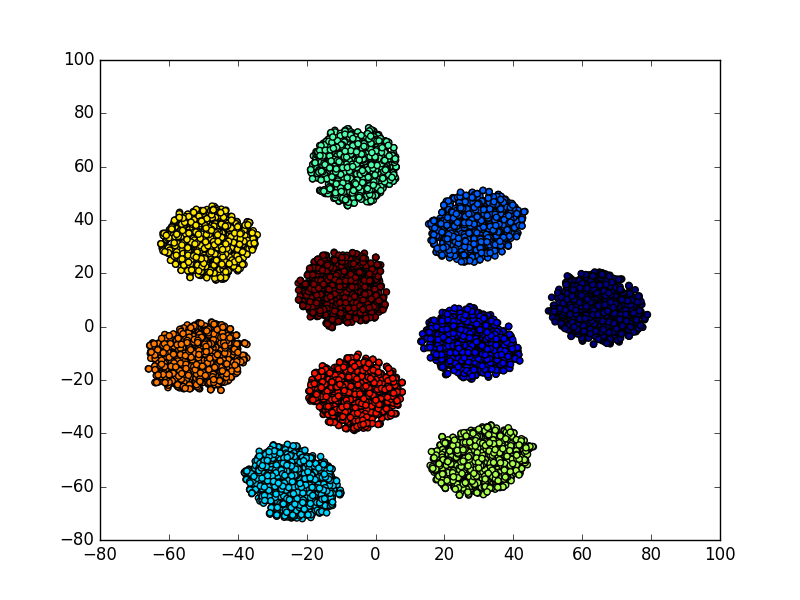} 
\caption{Linear generator with $z$ one-hot encoded}
\label{lin}
\end{subfigure}
\caption{TSNE visualization of latent vectors. Linear
Generator recovers clusters, suggesting that representation
power is not a bottleneck.}
\label{lin-plot}
\end{figure} 

\section{Discrete-Continuous Prior}

\subsection{Background}
Generative adversarial networks consist of two components, the generator $\mathcal{G}$ and the discriminator $\mathcal{D}$. Both $\mathcal{G}$ and $\mathcal{D}$ are usually implemented as neural networks parameterized by $\Theta_{G}$ and $\Theta_{D}$ respectively. The generator can also be considered to be a mapping from latent space to the data space which we denote as $\mathcal{G} : \mathcal{Z} \mapsto \mathcal{X}$. The discrimator defines a mapping from the data space to a real value which can correspond to the probability of the sample being real, $\mathcal{D} : \mathcal{X} \mapsto \mathbb{R} $. The GAN training sets up a two player game between $\mathcal{G}$ and $\mathcal{D}$, which is defined by the minimax objective : $ \min_{\Theta_G} \max_{\Theta_{D}} \mathop{\mathbf{E}}_{x \sim \mathbb{P}^r_x} q(\mathcal{D}(x)) + \mathop{\mathbf{E}}_{z \sim \mathbb{P}_z} q(1-\mathcal{D}(\mathcal{G}(z)))$, where $\mathbb{P}^r_x$ is the distribution of real data samples, $\mathbb{P}_z$ is the prior noise distribution on the latent space and $q(.)$ is the quality function. For vanilla GAN, $q(x) = \log x$, and for Wasserstein GAN (WGAN) $q(x) = x$.  We also denote the distribution of generated samples $x_g$ as $\mathbb{P}^g_x$. The discriminator and the generator are optimized alternatively so that at the end of training $\mathbb{P}^g_x$ matches $\mathbb{P}^r_x$. 

\subsection{Vanilla GAN does not cluster well in the latent space}
One possible way to cluster using a GAN is to back-propagate the data into the latent space (using back-propogation decoding \cite{lipton}) and cluster the latent space. However, this method usually leads to very bad results (see Fig. \ref{recon-mnist} for clustering results on MNIST). The key reason is that, if indeed, back-propagation succeeds, then the back-projected data distribution should look similar to the latent space distribution, which is typically chosen to be a Gaussian or uniform distribution, and we cannot expect to cluster in that space. Thus even though the latent space may contain full information about the data, the distance geometry in the latent space does not reflect the inherent clustering. In \cite{deligan}, the authors sampled from a Gaussian mixture prior and obtained diverse samples even in limited data regimes. However, even GANs with a Gaussian mixture failed to cluster, as shown in \ref{recon-mnist}(c). As observed by the authors of DeLiGAN, Gaussian components tend to ‘crowd’ and become redundant. Lifting the space using categorical variables could solve the problem
effectively. But continuity in latent space is traditionally viewed to be a pre-requisite for the objective of good interpolation. In other words, interpolation seems to be at loggerheads with the clustering objective. We demonstrate in this paper how ClusterGAN can obtain \textbf{good interpolation and good clustering} simultaneously.  

\begin{figure}
\begin{subfigure}[b]{0.50\textwidth}
\includegraphics[width=\textwidth]{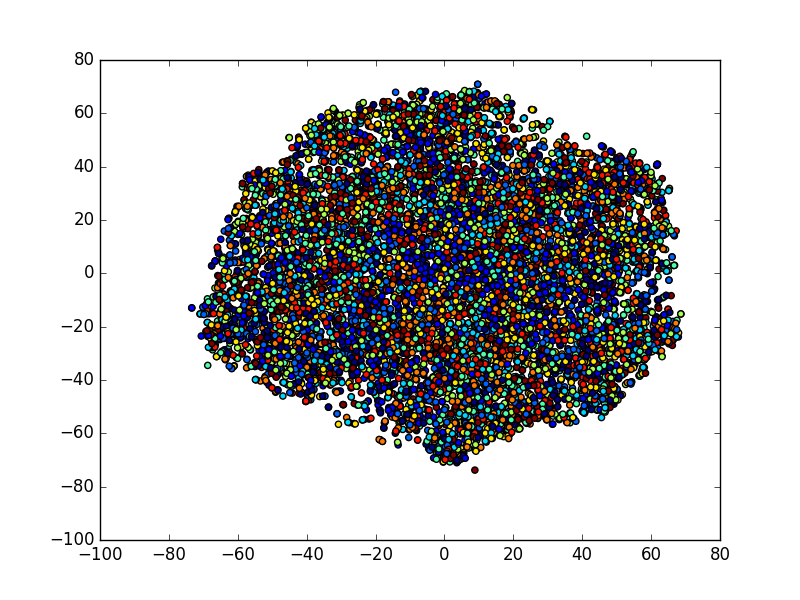} 
\caption{$z \sim $ Uniform}
\label{uniform}
\end{subfigure}
\begin{subfigure}[b]{0.50\textwidth}
\includegraphics[width=\textwidth]{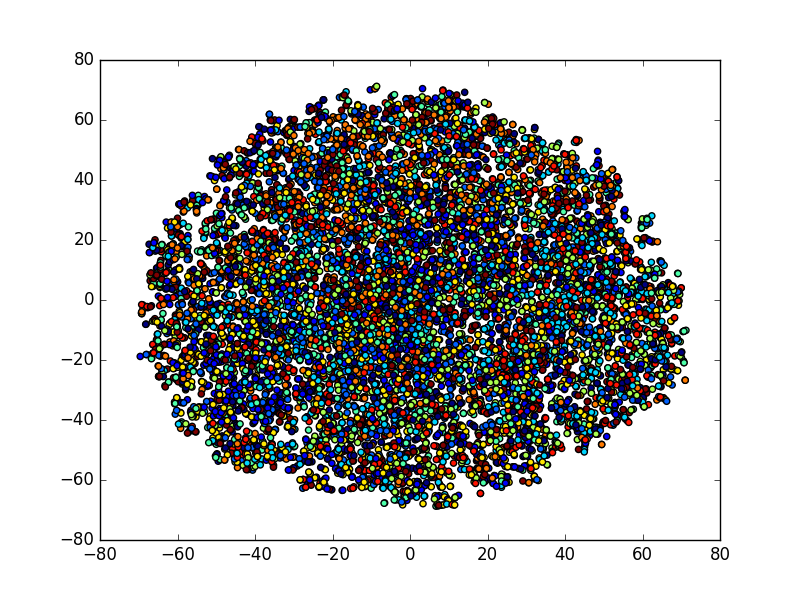} 
\caption{$z \sim $ Normal}
\label{normal}
\end{subfigure}
\begin{subfigure}[b]{0.50\textwidth}
\includegraphics[width=\textwidth]{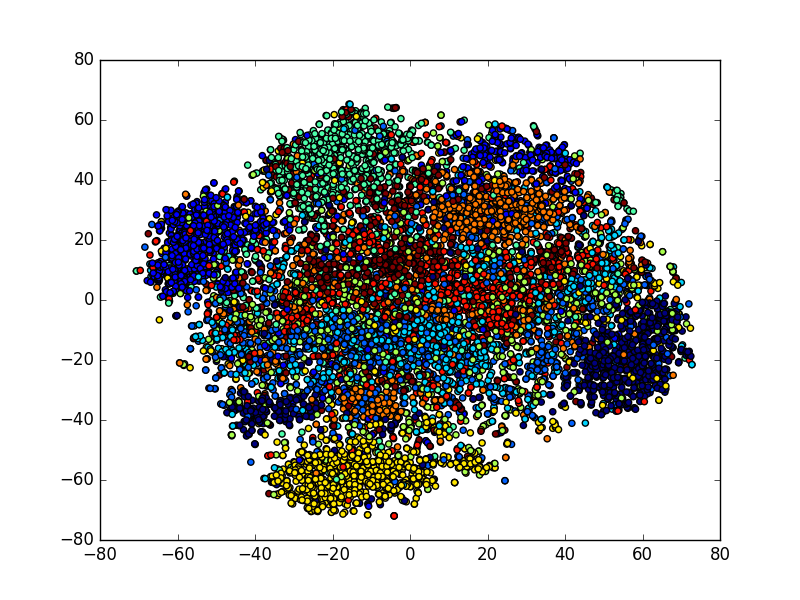} 
\caption{$z \sim $ Gaussian Mix}
\label{mix-gauss}
\end{subfigure}
\begin{subfigure}[b]{0.50\textwidth}
\includegraphics[width=\textwidth]{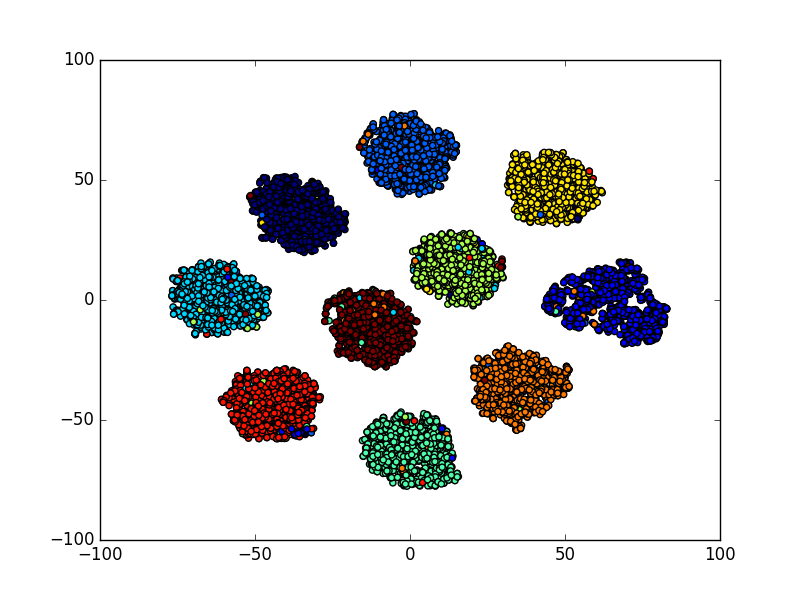} 
\caption{$z \sim (z_n, z_c)$  }
\label{one-hot}
\end{subfigure}
\caption{TSNE visualization of latent vectors for GANs trained with different priors on MNIST.}
\label{recon-mnist}
\end{figure} 

\subsection{Sampling from Discrete-Continuous Mixtures}

In ClusterGAN, we sample from a prior that consists of normal random variables cascaded with one-hot encoded vectors. To be more precise $z = (z_n, z_c), z_n \sim \mathcal{N}(0, \sigma^2 I_{d_n}), z_c = e_k, k \sim \mathcal{U}\{ 1, 2, \ldots, K\}$, $e_k$ is the $k^{th}$ elementary vector in $\mathbb{R}^{K}$ and $K$ is the number of clusters in the data. In addition, we need to choose $\sigma$ in such a way that the one-hot vector provides sufficient signal to the GAN training that leads to each mode only generating samples from a corresponding class in the original data. To be more precise, we chose $\sigma = 0.10$ in all our experiments so that each dimension of the normal latent variables, $z_{n,j} \in (-0.6, 0.6) << 1.0 \, \, \forall j$ with high probability. Small variances $\sigma$ are chosen to ensure the clusters in $\mathcal{Z}$ space are separated. Hence this prior naturally enables us to design an algorithm that clusters in the latent space. 

\subsection{Modified Backpropagation Based Decoding}

Previous works \cite{creswell} \cite{lipton} have explored solving an optimization problem in $z$ to recover the latent vectors, $z^* =  \mathop{\arg \min}_z  \mathcal{L} (\mathcal{G}(z), x) + \lambda \| z \|_{p}$, where $\mathcal{L}$ is some suitable loss function and $\|\cdot\|_p$ denotes the norm. This approach is insufficient for clustering with traditional latent priors even if backpropagation was lossless and recovered accurate latent vectors. To make the situation worse, the  optimization problem above is non-convex in $z$ ($\mathcal{G}$ being implemented as a neural network) and can obtain different embeddings in the $\mathcal{Z}$ space based on initialization. Some of the approaches to address this issue could be multiple restarts with different initialiations to obtain $z^*$, or stochastic clipping of $z$ at each iteration step. None of these lead to clustering, since they do not address the root problem of sampling from separated manifolds in $\mathcal{Z}$.  But our sampling procedure naturally gives way to such an algorithm. We use $\mathcal{L}(\mathcal{G}(z), x) = \| \mathcal{G}(z) - x \|_1$. Since we sample from a normal distrubution, we use the regularizer $\| z_n \|_2^2$, penalizing only the normal variables. We use $K$ restarts, each sampling $z_c$ from a different one-hot component and optimize with respect to only the normal variables, keeping $z_c$  fixed. Adam \cite{adam} is used for the updates during Backprop decoding. Formally, Algorithm \ref{alg:bpd} summarizes the approach.

\begin{algorithm}
\DontPrintSemicolon
\SetAlgoLined
\KwIn{Real sampler $x$, Generator function $\mathcal{G}$, Number of Clusters $K$, Regularization parameter $\lambda$, Adam iterations $\tau$  } 
\KwOut{Latent embedding $z^*$}
\For{$k \in \{ 1, 2, \ldots K \}$}{
Sample $z^0_n \sim \mathcal{N}(0, \sigma^2 I_{d_n})$ \;
Initialization $z_k^0 \gets ( z^0_n, e_k)$ ($e_k$ is $k^{th}$ elementary unit vector in $K$ dimensions)

\For{$t \in \{ 1, 2, \ldots \tau \}$}{
    Obtain the gradient of loss function $ g \gets \nabla_{z_n} \left( \| \mathcal{G}(z_k^{t-1}) - x \|_1 + \lambda \| z_n^{t-1}\|_2 \right)$\;
    Update $z_n^t$ using $g$ with Adam iteration to minimize loss.\;
    Clipping of $z_n^t$, i.e., $z_n^t \gets \mathcal{P}_{[-0.6, 0.6]}(z_n^t)$ \;
    $z_k^t \gets (z_n^t, e_k)$\;
    }
    Update $z^*$ if $z_k^{\tau}$ has lowest loss obtained so far.  
}
\Return{$z^*$}\;
\caption{{\sc Decode\_latent}}
\label{alg:bpd}
\end{algorithm}

\subsection{Linear Generator clusters perfectly}
The following lemma suggests that with discrete-continuous mixtures, we need only linear generation to generate mixture of Gaussians in the generated space.
\begin{lemma}
Clustering with only $z_n$ cannot recover a mixture of gaussian data in the linearly generated space. Further $\exists$ a linear $G(\cdot)$ mapping discrete-continuous mixtures to a mixture of Gaussians.
\end{lemma}
\begin{proof}
If latent space has only the continuous part, $z_n\sim \mathcal{N}(0, \sigma^2 I_{d_n})$, then by the linearity property, any linear generation can only produce Gaussian in the generated space. Now we show there exists a $G(\cdot)$ mapping discrete-continuous mixtures to the generate data $X\sim \mathcal{N}(\mu_\omega, \sigma^2 I_{d_n})$, where $\omega \sim \mathcal{U}\{ 1, 2, \ldots, K\}$ ($K$ is the number of mixtures). This is possible if we let $z_n \sim \mathcal{N}(0, \sigma^2 I_{d_n}), z_c = e_k, k \sim \mathcal{U}\{ 1, 2, \ldots, K\}$ and $G(z_n, z_c)=z_n+Az_c$, $A=diag[\mu_1, \cdots, \mu_K]$ being a $K\times K$ diagonal matrix with diagonal entries as the means $\mu_i$.
\end{proof}

To illustrate this lemma, and hence the drawback of traditional priors $\mathbb{P}_z $ for clustering, we performed a simple experiment. The real samples are drawn from a mixture of $10$ Gaussians in $\mathbb{R}^{100}$. The means of the Gaussians are sampled from $\mathcal{U}(-0.3, 0.3)^{100}$ and the variance of each component is fixed at $\sigma = 0.12$. We trained a GAN with $z \sim \mathcal{N}(0, I_{10})$ where the generator is a multi-layer perceptron with two hidden layers of $256$ units each. For comparison, we also trained a GAN with $z$ sampled from one-hot encoded normal vectors, the dimension of categorical variable being $10$. The generator for this GAN consisted of a linear mapping $W \in \mathbb{R}^{100 \times 10}$, such that $x = Wz$. After training, the latent vectors are recovered using Algorithm \ref{alg:bpd} for the linear generator, and $10$ restarts with random initializations for the non-linear generator. Even for this toy setup, the linear generator perfectly clustered the latent vectors (Acc. = 1.0, NMI = 1.0, ARI = 1.0), but the non-linear generator performed poorly (Acc. = 0.73, NMI = 0.75, ARI = 0.60) (Figure \ref{lin-plot}). The situation becomes worse for real datasets such as MNIST when we trained a GAN using latent vectors drawn from uniform, normal or a mixture of Gaussians. None of these configurations succeeded in clustering in the latent space as shown in Figure \ref{recon-mnist}. 

\begin{figure}[h]
\begin{subfigure}[b]{0.24\textwidth}
\includegraphics[width=\textwidth]{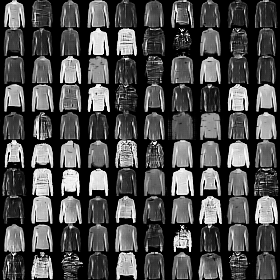} 
\end{subfigure}
\begin{subfigure}[b]{0.24\textwidth}
\includegraphics[width=\textwidth]{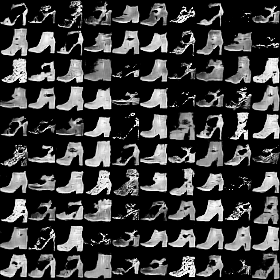} 
\end{subfigure}
\begin{subfigure}[b]{0.24\textwidth}
\includegraphics[width=\textwidth]{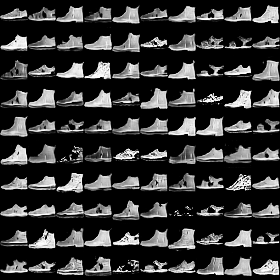} 
\end{subfigure}
\begin{subfigure}[b]{0.24\textwidth}
\includegraphics[width=\textwidth]{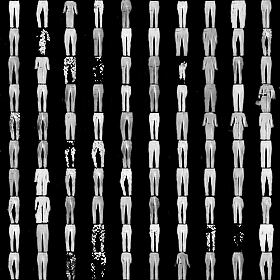} 
\end{subfigure}
\caption{Fashion items generated from distinct modes : Fashion-MNIST}
\label{fashion-mode}
\end{figure} 

\begin{figure*}[h]
\begin{subfigure}[b]{0.24\textwidth}
\includegraphics[width=\textwidth]{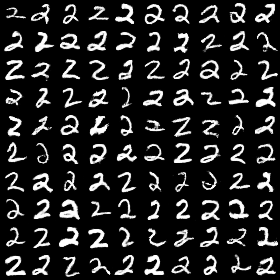} 
\end{subfigure}
\begin{subfigure}[b]{0.24\textwidth}
\includegraphics[width=\textwidth]{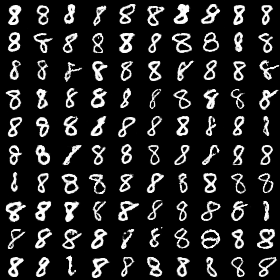} 
\end{subfigure}
\begin{subfigure}[b]{0.24\textwidth}
\includegraphics[width=\textwidth]{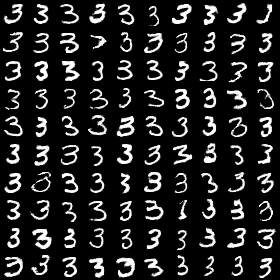} 
\end{subfigure}
\begin{subfigure}[b]{0.24\textwidth}
\includegraphics[width=\textwidth]{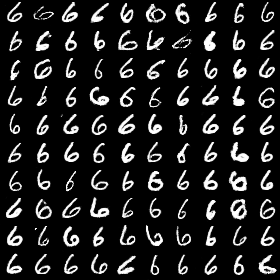} 
\end{subfigure}
\caption{Digits generated from distinct modes : MNIST}
\label{mnist-mode}
\end{figure*}

\subsection{Separate Modes for distinct classes in the data}

It was surprising to find that trained in a purely unsupervised manner without additional loss terms, each one-hot encoded component generated points from a specific class in the original data. For instance, $z = (z_n, e_k)$ generated a particular digit $\pi(k)$ in MNIST, for multiple samplings  of $z_n \sim  \mathcal{N}(0, \sigma^2 I_{d_n})$  ($\pi$ denotes a permutation). This was a necessary first step for the success of Algorithm \ref{alg:bpd}. We also quantitatively evaluated the modes learnt by the GAN by using a supervised classifier for MNIST. The supervised classifier had a test accuracy of $99.2\%$, so it had high reliability of distinguishing the digits. We sample from a mode $k$ and generate a digit $x_g$. It is then classified by the classifier as $\hat{y}$. From this pair $(k, \hat{y})$, we can map each mode to a digit and compute the accuracy of digit $\hat{y}$ being generated from mode $k$. This is denoted as Mode Accuracy. Each digit sample $x_r$  with label $y$ can be decoded in the latent space by Algorithm $\ref{alg:bpd}$ to obtain $z$. Now $z$ can be used to generate $x_g$, which when passed through the classifier gives the label $\hat{y}$. The pair $(y, \hat{y})$ must be equal in the ideal case and this accuracy is denoted as Reconstruction Accuracy. Finally, all the mappings of points in the same class in $\mathcal{X}$ space should have the same one-hot encoding when embedded in $\mathcal{Z}$ space. This defines the Cluster Accuracy. This metholodgy can be extended to quantitatively evaluate mode generation for other datasets also, provided there is a reliable classifier. For MNIST, we obtained Mode Accuracy of $0.97$, Reconstruction Accuracy of $0.96$ and Cluster Accuracy of $0.95$. Some of the modes in Fashion-MNIST and MNIST are shown in Figures \ref{fashion-mode} and \ref{mnist-mode}, respectively. Supplementary materials contain the images from all modes in these two datasets.

\subsection{Interpolation in latent space is preserved}

The latent space in a traditional GAN with Gaussian latent distribution enforces that different classes are continuously scattered in the latent space, allowing nice inter-class interpolation, which is a key strength of GANs. In ClusterGAN, the latent vector $z_c$ is sampled with a one-hot distribution and in order to interpolate across the classes, we will have to sample from a convex combination on the one-hot vector. While these vectors have never been sampled during the training process, we surprisingly observed very smooth inter-class interpolation in ClusterGAN. To demonstrate interpolation, we fixed the $z_n$ in two latent vectors with different $z_c$ components, say $z^{(1)}_{c}$ and $z^{(2)}_{c}$ and interpolated with the one-hot encoded part to give rise to new latent vectors $z = (z_n , \mu z^{(1)}_{c} + (1-\mu) z^{(2)}_{c}) , \mu \in [0, 1]$. As Figure \ref{interpol} illustrates, we observed a nice transition from one digit to another as well as across different classes in FashionMNIST. This demonstrates that ClusterGAN learns a very smooth manifold even on the untrained directions of the discrete-continuous distribution. We also show interpolations from a vanilla GAN trained with Gaussian prior as reference.

\begin{algorithm*}
\DontPrintSemicolon
\SetAlgoLined
\KwIn{Functions $\mathcal{G}$, $\mathcal{D}$ and $\mathcal{E}$, Regularization parameters $\beta_n$, $\beta_c$, learning rate $\eta$,  parameters $\Theta_G^t$, $\Theta_E^t$ }
\KwOut{$\Theta_G^{(t+1)}, \Theta_E^{(t+1)}$}
Sample ${z^{(i)}}_{i=1}^m$ from $ \mathbb{P}^z , z = (z_n, z_c)$ \; 
$g_{\Theta_G} \gets \nabla_{\Theta_G} \left( -\sum\limits_{i=1}^m q(\mathcal{D}(\mathcal{G}(z^{(i)})) + \beta_n \sum\limits_{i=1}^m \| z_n^{(i)} - \mathcal{E}(\mathcal{G}(z_n^{(i)})) \|_2^2  + \beta_c \sum\limits_{i=1}^m \mathcal{H} (z_c^{(i)} ,\mathcal{E}(\mathcal{G}(z_c^{(i)}))) \right) $  \;
$g_{\Theta_E} \gets \nabla_{\Theta_E}  \left( \beta_n \sum\limits_{i=1}^m \| z_n^{(i)} - \mathcal{E}(\mathcal{G}(z_n^{(i)})) \|_2^2  +   \beta_c \sum\limits_{i=1}^m \mathcal{H} (z_c^{(i)} ,\mathcal{E}(\mathcal{G}(z_c^{(i)}))) \right) $ \;
Update $\Theta_G$ using $(g_{\Theta_G}, \Theta_G^t)$ with Adam ; similarly for $\Theta_E$. \;
\Return{$\Theta_G, \Theta_E$}\;
\caption{{\sc Update\_Param}}
\label{alg:update}
\end{algorithm*}

\begin{figure*}[h]
\begin{subfigure}[b]{0.5\textwidth}
\includegraphics[width=\textwidth]{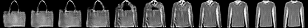} 
\end{subfigure}
\begin{subfigure}[b]{0.5\textwidth}
\includegraphics[width=\textwidth]{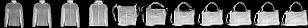} 
\end{subfigure}
\begin{subfigure}[b]{0.5\textwidth}
\includegraphics[width=\textwidth]{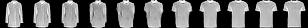} 
\end{subfigure}
\begin{subfigure}[b]{0.5\textwidth}
\includegraphics[width=\textwidth]{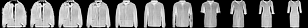} 
\end{subfigure}
\begin{subfigure}[b]{0.5\textwidth}
\includegraphics[width=\textwidth]{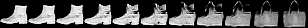} 
\end{subfigure}
\begin{subfigure}[b]{0.5\textwidth}
\includegraphics[width=\textwidth]{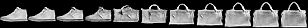} 
\end{subfigure}
\begin{subfigure}[b]{0.5\textwidth}
\includegraphics[width=\textwidth]{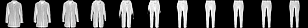} 
\end{subfigure}
\begin{subfigure}[b]{0.5\textwidth}
\includegraphics[width=\textwidth]{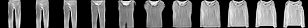} 
\end{subfigure}
\begin{subfigure}[b]{0.5\textwidth}
\includegraphics[width=\textwidth]{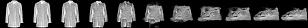} 
\end{subfigure}
\begin{subfigure}[b]{0.5\textwidth}
\includegraphics[width=\textwidth]{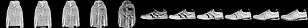} 
\end{subfigure}
\begin{subfigure}[b]{0.5\textwidth}
\includegraphics[width=\textwidth]{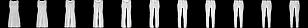} 
\end{subfigure}
\begin{subfigure}[b]{0.5\textwidth}
\includegraphics[width=\textwidth]{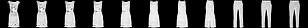} 
\end{subfigure}
\begin{subfigure}[b]{0.5\textwidth}
\includegraphics[width=\textwidth]{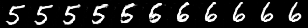} 
\end{subfigure}
\begin{subfigure}[b]{0.5\textwidth}
\includegraphics[width=\textwidth]{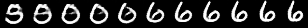} 
\end{subfigure}
\begin{subfigure}[b]{0.5\textwidth}
\includegraphics[width=\textwidth]{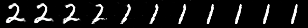} 
\end{subfigure}
\begin{subfigure}[b]{0.5\textwidth}
\includegraphics[width=\textwidth]{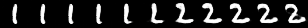} 
\end{subfigure}
\caption{Comparison of Latent Space Interpolation : ClusterGAN (left) and vanilla WGAN (right)}
\label{interpol}
\end{figure*} 

\begin{figure}[ht]
\includegraphics[scale = 0.75]{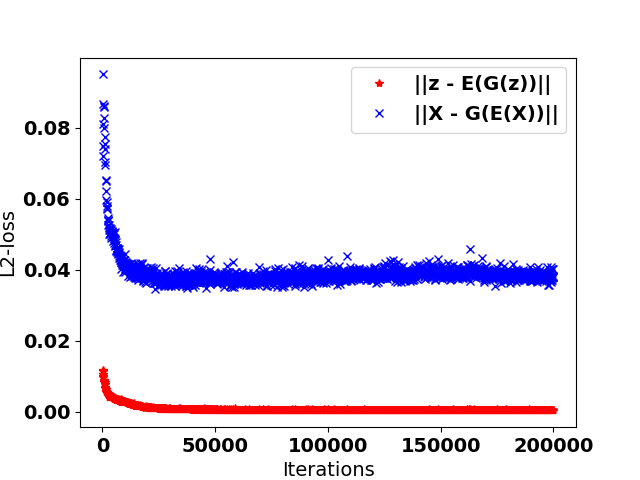} 
\caption{Decrease of Cycle Loss with iterations in MNIST. The mean square L2-distance $|| x - \mathcal{G}(\mathcal{E}(x))||$ was $0.038$ and $|| z - \mathcal{E}(\mathcal{G}(z))||$ was $0.0004$. Mean distances were computed on a test batch not used in training.}
\label{cycle_loss}
\end{figure}

\section{ClusterGAN}

Even though the above approach enables the GAN to cluster in the latent space, it may be able to perform even better if we had a clustering specific loss term in the minimax objective. For MNIST, digit strokes correspond well to the category in the data. But for more complicated datasets, we need to enforce structure in the GAN training. One way to ensure that is to enforce precise recovery of the latent vector. We therefore introduce an encoder $\mathcal{E} : \mathcal{X} \mapsto \mathcal{Z}$, a neural network parameterized by $\Theta_E$. The GAN objective now takes the following form:
\begin{multline}
 \min_{\Theta_G, \Theta_E} \max_{\Theta_{D}} \mathop{\mathbf{E}}_{x \sim \mathbb{P}^r_x} q(\mathcal{D}(x))+\mathop{\mathbf{E}}_{z \sim \mathbb{P}_z} q(1-\mathcal{D}(\mathcal{G}(z))) \\+\beta_n\mathop{\mathbf{E}}_{z \sim \mathbb{P}_z} \| z_n - \mathcal{E}(\mathcal{G}(z_n)) \|_2^2 + \beta_c\mathop{\mathbf{E}}_{z \sim \mathbb{P}_z} \mathcal{H} (z_c ,\mathcal{E}(\mathcal{G}(z_c))
\end{multline}

where $\mathcal{H}(.,.)$ is the cross-entropy loss. The relative magnitudes of the regularization coeficients $\beta_n$ and $\beta_c$ enable a flexible choice to vary the importance of preserving the discrete and continuous portions of the latent code.   One could imagine other variations of the regularization that map $\mathcal{E}(\mathcal{G}(z))$ to be close to the centroid of the respective cluster, for instance $\|\mathcal{E}(\mathcal{G}(z^{(i)})) - \mu^{c(i)}\|_2^2 $, in similar spirit as K-Means. The GAN training in this approach involves jointly updating the parameters of $\Theta_G$ and $\Theta_E$ (Algorithm \ref{alg:update}).

As shown in Figure \ref{cycle_loss}, in our architecture, both $x$ is close to $\mathcal{G}(\mathcal{E}(x))$ and $z$ is close to $\mathcal{E}(\mathcal{G}(z))$. Even though our architecture is optimizing for one type of cycle loss, both losses are small. The loss optimized for is even smaller.

\section{Experiments}

\begin{table}
\centering
\begin{tabular}{ |c|c|c|c|c|  }
 \Xhline{2.0pt}
  \textbf{Dataset} & \textbf{Algorithm} & \textbf{ACC} & \textbf{NMI} & \textbf{ARI}  \\
 \Xhline{2.0pt} 
 \multirow{4}{*}{Synthetic} & ClusterGAN &{\color{blue} 0.99} & {\color{blue}0.99} & {\color{blue}0.99} \\ \cline{2-5}  
 & Info-GAN & 0.88 & 0.75  & 0.74 \\ \cline{2-5}
 & GAN with bp & 0.95 & 0.85  & 0.88 \\ \cline{2-5}
 & GAN with Disc. $\phi$ & 0.99  & 0.98   & 0.98 \\ \cline{2-5}
 %& TSNE & 0.93 & 0.82  & 0.83  \\ \cline{2-5}
 & AGGLO. & 0.99 & 0.99  & 0.99  \\ \cline{2-5}
 & NMF  & 0.98 & 0.96 & 0.97 \\\cline{2-5}
 & SC  & 0.99 & 0.98 & 0.98 \\
 \Xhline{2.0pt}
 \multirow{5}{*}{MNIST} & ClusterGAN &{\color{blue} 0.95} & 0.89 & {\color{blue}0.89} \\ \cline{2-5}  
 & Info-GAN & 0.89 & 0.86  & 0.82 \\ \cline{2-5}
 & GAN with bp & 0.95 & {\color{blue}0.90}  & 0.89 \\ \cline{2-5}
 & GAN with Disc. $\phi$ &  0.70 & 0.62  &  0.52\\ \cline{2-5}
 & DCN & 0.83 &  0.81  & 0.75 \\  \cline{2-5}
 %& TSNE & 0.90 & 0.83  & 0.80  \\ \cline{2-5}
 & AGGLO. & 0.64 & 0.65  & 0.46  \\ \cline{2-5}
 & NMF  & 0.56 & 0.45 & 0.36 \\
 \Xhline{2.0pt}
 \multirow{4}{*}{Fashion-10} & ClusterGAN & {\color{blue} 0.63} & {\color{blue}0.64} & {\color{blue}0.50} \\ \cline{2-5} 
 & Info-GAN & 0.61 & 0.59  & 0.44 \\ \cline{2-5} 
 & GAN with bp & 0.56 & 0.53  & 0.37 \\ \cline{2-5}
 & GAN with Disc. $\phi$ & 0.43 & 0.37  &  0.23 \\ \cline{2-5}
 %& TSNE & 0.55 & 0.54  & 0.37  \\ \cline{2-5}
 & AGGLO. & 0.55 & 0.57  & 0.37  \\ \cline{2-5}
 & NMF  & 0.50 & 0.51 & 0.34 \\
 \Xhline{2.0pt}
 \multirow{4}{*}{Fashion-5} & ClusterGAN & {\color{blue}0.73}  & {\color{blue}0.59}  & {\color{blue}0.48} \\ \cline{2-5}  
 & Info-GAN & 0.71 & 0.56  & 0.45 \\ \cline{2-5}
 & GAN with bp & 0.73 & 0.54  & 0.45 \\ \cline{2-5}
 & GAN with Disc. $\phi$ & 0.67 & 0.49  & 0.40 \\ \cline{2-5}
 %& TSNE & 0.72 & {\color{blue}0.60}  & {\color{blue}0.51}  \\ \cline{2-5}
 & AGGLO. & 0.66 & 0.52  & 0.36  \\ \cline{2-5}
 & NMF  & 0.67 & 0.48 & 0.40 \\
 \Xhline{2.0pt}
 \multirow{4}{*}{10x\_73k} & ClusterGAN & {\color{blue} 0.81} & {\color{blue}0.73} & {\color{blue}0.67} \\ \cline{2-5}  
 & Info-GAN & 0.62 & 0.58  & 0.43 \\ \cline{2-5}
 & GAN with bp & 0.65 & 0.59  & 0.45 \\ \cline{2-5}
 & GAN with Disc. $\phi$ & 0.33 & 0.17  &  0.07 \\ \cline{2-5}
 %& TSNE & 0.66 & 0.65  & 0.51  \\ \cline{2-5}
 & AGGLO. & 0.63 & 0.58  & 0.40  \\ \cline{2-5}
 & NMF  & 0.71 & 0.69 & 0.53 \\ \cline{2-5}
 & SC  & 0.40 & 0.29 & 0.18 \\
 \Xhline{2.0pt} 
 \multirow{4}{*}{Pendigits} & ClusterGAN & {\color{blue}0.77} &  {\color{blue}0.73}  &  {\color{blue}0.65} \\ \cline{2-5}  
 & Info-GAN & 0.72 & 0.73  & 0.61 \\ \cline{2-5}
 & GAN with bp & 0.76 & 0.71  & 0.63 \\ \cline{2-5}
 & GAN with Disc. $\phi$ & 0.65 & 0.57  & 0.45 \\ \cline{2-5}
 & DCN & 0.72 & 0.69  & 0.56  \\ \cline{2-5}
 %& TSNE & {\color{blue}0.86} & {\color{blue}0.81}  & {\color{blue}0.74}  \\ \cline{2-5}
 & AGGLO. & 0.70 & 0.69  & 0.52  \\ \cline{2-5}
 & NMF  & 0.67 & 0.58 & 0.45 \\ \cline{2-5}
 & SC  & 0.70  & 0.69  &  0.52 \\
 \Xhline{2.0pt}
\end{tabular}
\caption{Comparison of clustering metrics across datasets} 
\label{table:compare_alg}
\end{table} 

\subsection{Datasets}

\textbf{Synthetic Data} The data is generated from a mixture of Gaussians with $4$ components in 2D, which constitutes the $\mathcal{Z}$ space. We generated $2500$ points from each Gaussian. The $\mathcal{X}$ space is obtained by a non-linear transformation : $x = f(\mathbf{U} \cdot f ( \mathbf{W} z))$, where $ \mathbf{W} \in \mathbb{R}^{10 \times 2} , \mathbf{U} \in \mathbb{R}^{100 \times 10}$ with $W_{i,j} \sim \mathcal{N}(0, 1), U_{i,j} \sim \mathcal{N} (0, 1)$. $f(\cdot)$ is the sigmoid function to introduce non-linearity. 

\textbf{MNIST}  It consists of $70$k images of digits ranging from $0$ to $9$. Each data sample is a $28 \times 28$ greyscale image. We used the DCGAN with conv-deconv layers, batch normalization and leaky relu activations, the details of which are available in the Supplementary material. 

\textbf{Fashion-MNIST (10 and 5 classes)} This dataset has the same number of images with the same image size as MNIST, but it is fairly more complicated. Instead of digits, it consists of various types of fashion products. Supervised methods achieve lower accuracy than MNIST on this dataset. For training a GAN, we used the same architecture as MNIST for this dataset. We also merged some categories which were similar to form a separate 5-class dataset. The five groups were as follows : \{Tshirt/Top, Dress\}, \{Trouser\}, \{Pullover, Coat, Shirt\}, \{Bag\},  \{Sandal, Sneaker, Ankle Boot\}.

\textbf{10x\_73k} Even though GANs have achieved unprecedented success in generating realistic images, it is not clear whether they can be equally effective for other types of data. In this experiment, we trained a GAN to cluster cell types from a single cell RNA-seq counts matrix. Moreover, computer vision might have ample supply of labelled images, obtaining labels for some fields,  for instance biology, is extremely costly and laborious. Thus, unsupervised clustering of data is truly a necessity for this domain. The dataset consists of RNA-transcript counts of $73233$ data points belonging to $8$ different cell types \cite{zheng}. To reduce the dimension of the data, we selected $720$ highest variance genes across the cells. The entries of the counts matrix $\mathbf{C}$ are first tranformed as $\log_2(1 + C_{ij})$ and then divided by the maximum entry of the transformation to obtain values in the range of $[0, 1]$. One of the major challenges in this data is sparsity. Even after subselection of genes based on variance, the data matrix was close to $40 \%$ zero entries.

\textbf{Pendigits} It is a very different dataset that consists of a time series of $\{ x_t, y_t \}_{t=1}^T $ coordinates. The points are sampled as writers write digits on a pressure sensitive tablet. The total number of datapoints is $10992$, and consists of $10$ classes, each for a digit. It provided a unique challenge of training GANs for point cloud data.

%For Synthetic dataset, 10x\_73k and Pendigits, we used fully connected neural networks for the Generator, Encoder and Discriminator functions. %
In all our experiments in this paper, we used an improved variant (WGAN-GP) which includes a gradient penalty \cite{gulrajani}. Using cross-validation for selecting hyperparameters is not an option in purely unsupervised problems due to absence of labels. We adapted standard architectures for the datasets \cite{infogan} and avoided data specific tuning as much as possible. Some choices of regularization parameters $\lambda = 10$, $\beta_n = 10$, $\beta_r = 10$ worked well across all datasets.

\subsection{Evaluation}
Since clustering is an unsupervised problem, we ensured that all the algorithms are oblivious to the true labels unlike a supervised framework like conditional GAN \cite{mirza}. We compared ClusterGAN with other possible GAN based clustering approaches we could conceive. 

\begin{table}
\centering
\begin{tabular}{ |c|c|c|c|c|  }
 \Xhline{2.0pt}
 \textbf{Dataset} & \multicolumn{4}{|c|}{\textbf{Algorithm}} \\
 \cline{2-5}
   & Cluster & WGAN & WGAN & Info \\
   & GAN & (Normal) & (One-Hot) & GAN \\
 \hline
 MNIST & {\color{blue}0.81} & 0.88 & 0.94 & 1.88 \\
 \hline
 Fashion & {\color{blue}0.91} & 0.95 & 6.14 & 11.04 \\
 %\hline
 %10x\_73k & 2.50 & {\color{blue}2.02} & 2.24 & 25.59 \\
 %\hline
 %Pendigits & 9.56 & {\color{blue}6.45} & 13.44 & 87.80 \\
 \Xhline{2.0pt}
\end{tabular}
\caption{Comparison of Frechet Inception Distance (FID) (Lower distance is better)} 
\label{table:compare_fid}
\end{table} 

\begin{table}
\centering
\begin{tabular}{ |c|c|c|c|c|  }
 \Xhline{2.0pt}
 \multicolumn{5}{|c|}{Dataset : MNIST, Algorithm : ClusterGAN} \\
 \multicolumn{5}{|c|}{\textbf{ACC}} \\
 \hline
 $\textbf{K = 7}$ & $\textbf{K = 9}$ & $\textbf{K = 10}$ & $\textbf{K = 11}$ & $\textbf{K = 13}$ \\
 \hline
 0.72 & 0.80 & {\color{blue}0.95} & 0.90 & 0.84\\
 \Xhline{2.0pt}
\end{tabular}
\caption{Robustness to Cluster Number $K$} 
\label{table:robust_K}
\end{table} 

Algorithm \ref{alg:bpd} + K-Means is denoted as ``GAN with bp''. For InfoGAN,  we used $\mathop{{\arg\max}_c} \mathbb{P}(c \mid x)$ as an inferred cluster label for $x$. Further, the features $\phi(x)$ in the last layer of the Discriminator could contain some class-specific discriminating features for clustering. So we used Kmeans on $\phi(x)$ to cluster, denoted as ``GAN with Disc. $\phi$''. We also included clustering results from Non-negative matrix Factorization (NMF) \cite{lee}, Aggolomerative Clustering (AGGLO)  \cite{agglo} and Spectral Clustering (SC). AGGLO with Euclidean affinity score and ward linkage gave best results. NMF had both l-1 and l-2 regularization, initialized with Non-negative Double SVD and used KL-divergence loss. SC had rbf kernel for affinity. We reported normalized mutual information (NMI), adjusted Rand index (ARI), and clustering purity (ACC). Since DCN has been shown to outperform various deep-learning based clustering algorithms, we reported its metrics from the paper \cite{yang} for MNIST and Pendigits. We found DCN to be very sensitive to hyperparameter choice, architecture and learning rates and could not obtain reasonable results from it on the other datasets. But we outperformed DCN results on MNIST and Pendigits dataset \footnote{For all baselines and GAN variants, Table \ref{table:compare_alg} reports metrics for the model with best validation purity from $5$ runs.}. 

Since clustering metrics do not reveal the quality of generated samples from a GAN, we report the Frechet Inception Distance (FID) \cite{heusel} for the image datasets. We found that ClusterGAN achives good clustering without compromising sample quality as shown in Table \ref{table:compare_fid}.

In all datasets, we provided the true number of clusters to all algorithms. In addition, for MNIST, Table \ref{table:robust_K} provides the clustering performance of ClusterGAN as number of clusters is varied. Overestimates do not severely hurt ClusterGAN; but underestimate does. 

\subsection{Scalability to Large Number of Clusters} 
We ran ClusterGAN on Coil-20 ($N = 1440$, $K = 20$) and Coil-100 ($N = 7200, K = 100$) datasets, where $N$ is the number of Data points. ClusterGAN could obtain good clusters even with such high value of $K$. These data sets were particularly difficult for GAN training with only a few thousand data points. Yet, we found similar behavior as MNIST / Fashion-MNIST emerging here as well. Distinct modes generated distinct 3D-objects along with rotations as shown in Figure \ref{coil-mode}. 

\begin{figure*}[ht]
\centering
\begin{subfigure}[b]{0.24\textwidth}
\includegraphics[width=\textwidth]{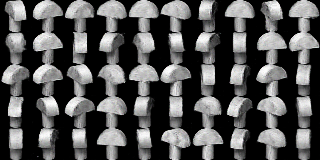} 
\end{subfigure}
\begin{subfigure}[b]{0.24\textwidth}
\includegraphics[width=\textwidth]{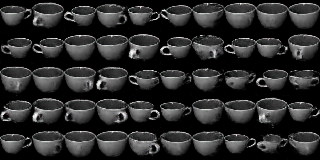} 
\end{subfigure}
\begin{subfigure}[b]{0.24\textwidth}
\includegraphics[width=\textwidth]{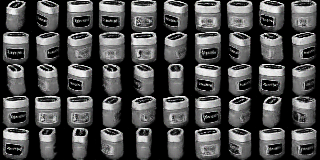} 
\end{subfigure}
\begin{subfigure}[b]{0.24\textwidth}
\includegraphics[width=\textwidth]{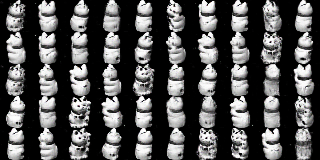} 
\end{subfigure}
\begin{subfigure}[b]{0.24\textwidth}
\includegraphics[width=\textwidth]{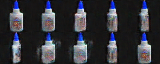} 
\end{subfigure}
\begin{subfigure}[b]{0.24\textwidth}
\includegraphics[width=\textwidth]{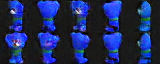} 
\end{subfigure}
\begin{subfigure}[b]{0.24\textwidth}
\includegraphics[width=\textwidth]{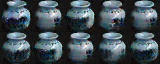} 
\end{subfigure}
\begin{subfigure}[b]{0.24\textwidth}
\includegraphics[width=\textwidth]{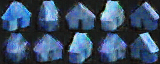} 
\end{subfigure}
\caption{Scalability of ClusterGAN to large number of clusters : Modes of Coil-20 (above) and Coil-100 (below)}
\label{coil-mode}
\end{figure*}

\section{Discussion and Future Work}

In this work, we discussed the drawback of training a GAN with traditional prior latent distributions for clustering and considered discrete-continuous mixtures for sampling noise variables. We proposed ClusterGAN, an architecture that enables clustering in the latent space. Comparison with clustering baselines on varied datasets using ClusterGAN illustrates that GANs can be suitably adapted for clustering. Future directions can explore better data-driven priors for the latent space. Another possibility is to improve results for problems that have a sparse generative structure such as compressed sensing.
%\nolinenumbers

%This is where your bibliography is generated. Make sure that your .bib file is actually called library.bib
\bibliography{Ref}

\begin{thebibliography}{10}

\bibitem{andrew}
Galen Andrew, Raman Arora, Jeff Bilmes, and Karen Livescu.
\newblock Deep canonical correlation analysis.
\newblock In {\em International Conference on Machine Learning}, pages
  1247--1255, 2013.

\bibitem{bengio}
Yoshua Bengio, Aaron Courville, and Pascal Vincent.
\newblock Representation learning: A review and new perspectives.
\newblock {\em IEEE transactions on pattern analysis and machine intelligence},
  35(8):1798--1828, 2013.

\bibitem{bora}
Ashish Bora, Ajil Jalal, Eric Price, and Alexandros~G. Dimakis.
\newblock Compressed sensing using generative models.
\newblock In {\em Proceedings of the 34th International Conference on Machine
  Learning, {ICML} 2017, Sydney, NSW, Australia, 6-11 August 2017}, pages
  537--546, 2017.

\bibitem{infogan}
Xi~Chen, Yan Duan, Rein Houthooft, John Schulman, Ilya Sutskever, and Pieter
  Abbeel.
\newblock Infogan: Interpretable representation learning by information
  maximizing generative adversarial nets.
\newblock In {\em Advances in Neural Information Processing Systems}, pages
  2172--2180, 2016.

\bibitem{coates}
Adam Coates, Andrew Ng, and Honglak Lee.
\newblock An analysis of single-layer networks in unsupervised feature
  learning.
\newblock In {\em Proceedings of the fourteenth international conference on
  artificial intelligence and statistics}, pages 215--223, 2011.

\bibitem{creswell}
Antonia Creswell and Anil~A Bharath.
\newblock Inverting the generator of a generative adversarial network (ii).
\newblock {\em arXiv preprint arXiv:1802.05701}, 2018.

\bibitem{donahue}
Jeff Donahue, Philipp Kr{\"a}henb{\"u}hl, and Trevor Darrell.
\newblock Adversarial feature learning.
\newblock {\em arXiv preprint arXiv:1605.09782}, 2016.

\bibitem{dumoulin}
Vincent Dumoulin, Ishmael Belghazi, Ben Poole, Olivier Mastropietro, Alex Lamb,
  Martin Arjovsky, and Aaron Courville.
\newblock Adversarially learned inference.
\newblock {\em arXiv preprint arXiv:1606.00704}, 2016.

\bibitem{goodfellow}
Ian Goodfellow, Jean Pouget-Abadie, Mehdi Mirza, Bing Xu, David Warde-Farley,
  Sherjil Ozair, Aaron Courville, and Yoshua Bengio.
\newblock Generative adversarial nets.
\newblock In {\em Advances in neural information processing systems}, pages
  2672--2680, 2014.

\bibitem{gulrajani}
Ishaan Gulrajani, Faruk Ahmed, Martin Arjovsky, Vincent Dumoulin, and Aaron~C
  Courville.
\newblock Improved training of wasserstein gans.
\newblock In {\em Advances in Neural Information Processing Systems}, pages
  5769--5779, 2017.

\bibitem{deligan}
Swaminathan Gurumurthy, Ravi~Kiran Sarvadevabhatla, and R~Venkatesh Babu.
\newblock Deligan: Generative adversarial networks for diverse and limited
  data.

\bibitem{halkidi}
Maria Halkidi, Yannis Batistakis, and Michalis Vazirgiannis.
\newblock On clustering validation techniques.
\newblock {\em Journal of intelligent information systems}, 17(2-3):107--145,
  2001.

\bibitem{heusel}
Martin Heusel, Hubert Ramsauer, Thomas Unterthiner, Bernhard Nessler, and Sepp
  Hochreiter.
\newblock Gans trained by a two time-scale update rule converge to a local nash
  equilibrium.
\newblock In {\em Advances in Neural Information Processing Systems}, pages
  6626--6637, 2017.

\bibitem{ilyas}
Andrew Ilyas, Ajil Jalal, Eirini Asteri, Constantinos Daskalakis, and
  Alexandros~G Dimakis.
\newblock The robust manifold defense: Adversarial training using generative
  models.
\newblock {\em arXiv preprint arXiv:1712.09196}, 2017.

\bibitem{ji}
Pan Ji, Tong Zhang, Hongdong Li, Mathieu Salzmann, and Ian Reid.
\newblock Deep subspace clustering networks.
\newblock In {\em Advances in Neural Information Processing Systems}, pages
  23--32, 2017.

\bibitem{adam}
Diederik~P Kingma and Jimmy Ba.
\newblock Adam: A method for stochastic optimization.
\newblock {\em arXiv preprint arXiv:1412.6980}, 2014.

\bibitem{kingma}
Diederik~P Kingma and Max Welling.
\newblock Auto-encoding variational bayes.
\newblock {\em arXiv preprint arXiv:1312.6114}, 2013.

\bibitem{lee}
Daniel~D Lee and H~Sebastian Seung.
\newblock Learning the parts of objects by non-negative matrix factorization.
\newblock {\em Nature}, 401(6755):788, 1999.

\bibitem{lipton}
Zachary~C Lipton and Subarna Tripathi.
\newblock Precise recovery of latent vectors from generative adversarial
  networks.
\newblock {\em arXiv preprint arXiv:1702.04782}, 2017.

\bibitem{maaten}
Laurens van~der Maaten and Geoffrey Hinton.
\newblock Visualizing data using t-sne.
\newblock {\em Journal of machine learning research}, 9(Nov):2579--2605, 2008.

\bibitem{makhzani}
Alireza Makhzani, Jonathon Shlens, Navdeep Jaitly, Ian Goodfellow, and Brendan
  Frey.
\newblock Adversarial autoencoders.
\newblock {\em arXiv preprint arXiv:1511.05644}, 2015.

\bibitem{mika}
Sebastian Mika, Bernhard Sch{\"o}lkopf, Alex~J Smola, Klaus-Robert M{\"u}ller,
  Matthias Scholz, and Gunnar R{\"a}tsch.
\newblock Kernel pca and de-noising in feature spaces.
\newblock In {\em Advances in neural information processing systems}, pages
  536--542, 1999.

\bibitem{mirza}
Mehdi Mirza and Simon Osindero.
\newblock Conditional generative adversarial nets.
\newblock {\em arXiv preprint arXiv:1411.1784}, 2014.

\bibitem{ngspec}
Andrew~Y Ng, Michael~I Jordan, and Yair Weiss.
\newblock On spectral clustering: Analysis and an algorithm.
\newblock In {\em Advances in neural information processing systems}, pages
  849--856, 2002.

\bibitem{radford}
Alec Radford, Luke Metz, and Soumith Chintala.
\newblock Unsupervised representation learning with deep convolutional
  generative adversarial networks.
\newblock {\em arXiv preprint arXiv:1511.06434}, 2015.

\bibitem{santurkar}
Shibani Santurkar, David Budden, and Nir Shavit.
\newblock Generative compression.
\newblock {\em arXiv preprint arXiv:1703.01467}, 2017.

\bibitem{tolstikhin}
Ilya Tolstikhin, Olivier Bousquet, Sylvain Gelly, and Bernhard Schoelkopf.
\newblock Wasserstein auto-encoders.
\newblock {\em arXiv preprint arXiv:1711.01558}, 2017.

\bibitem{vincent}
Pascal Vincent, Hugo Larochelle, Isabelle Lajoie, Yoshua Bengio, and
  Pierre-Antoine Manzagol.
\newblock Stacked denoising autoencoders: Learning useful representations in a
  deep network with a local denoising criterion.
\newblock {\em Journal of Machine Learning Research}, 11(Dec):3371--3408, 2010.

\bibitem{xiang}
Shiming Xiang, Feiping Nie, and Changshui Zhang.
\newblock Learning a mahalanobis distance metric for data clustering and
  classification.
\newblock {\em Pattern Recognition}, 41(12):3600--3612, 2008.

\bibitem{xie}
Junyuan Xie, Ross Girshick, and Ali Farhadi.
\newblock Unsupervised deep embedding for clustering analysis.
\newblock In {\em International conference on machine learning}, pages
  478--487, 2016.

\bibitem{yang}
Bo~Yang, Xiao Fu, Nicholas~D. Sidiropoulos, and Mingyi Hong.
\newblock Towards k-means-friendly spaces: Simultaneous deep learning and
  clustering.
\newblock In {\em Proceedings of the 34th International Conference on Machine
  Learning, {ICML} 2017, Sydney, NSW, Australia, 6-11 August 2017}, pages
  3861--3870, 2017.

\bibitem{agglo}
Wei Zhang, Xiaogang Wang, Deli Zhao, and Xiaoou Tang.
\newblock Graph degree linkage: Agglomerative clustering on a directed graph.
\newblock In {\em European Conference on Computer Vision}, pages 428--441.
  Springer, 2012.

\bibitem{zheng}
Grace~XY Zheng, Jessica~M Terry, Phillip Belgrader, Paul Ryvkin, Zachary~W
  Bent, Ryan Wilson, Solongo~B Ziraldo, Tobias~D Wheeler, Geoff~P McDermott,
  Junjie Zhu, et~al.
\newblock Massively parallel digital transcriptional profiling of single cells.
\newblock {\em Nature communications}, 8:14049, 2017.

\end{thebibliography}

%This defines the bibliographies style. Search online for a list of available styles.
\bibliographystyle{plain}

\newpage

\section{Supplementary Material}

\subsection{Additional Results}

\begin{figure*}[h]
\begin{subfigure}[b]{0.5\textwidth}
\includegraphics[width=\textwidth]{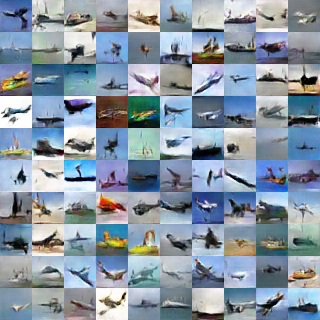} 
\end{subfigure}
\begin{subfigure}[b]{0.5\textwidth}
\includegraphics[width=\textwidth]{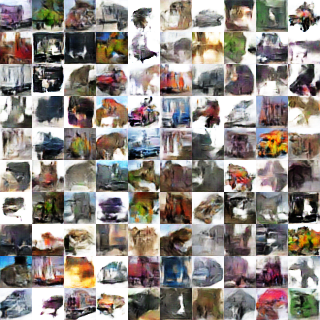} 
\end{subfigure}
\caption{Modes of CIFAR-10 pick up features that easily segregate categories, but may not correspond to dataset labels. Blue background (left) and predominant white background (right).}
\label{cifar10-mode}
\end{figure*}

We also ran ClusterGAN on CIFAR-10, which is a dataset with considerable intra-class variability. For CIFAR-10, the modes of ClusterGAN generate images based on a commonality (like blue background or white background), which may not be in correspondence with the labels. Blue background is present across images in airplane, bird, ship categories. Left to itself, a purely unsupervised model learns a representation that most easily separates the images into categories; but there is no impulsion to enforce higher order semantics of the labels. We believe that maximizing mutual information between an intermediate convolutional layer and the generated image would preserve the high-level semantics. We leave this exploration as future work.

\subsection{Hyperparameter and Architecture Details}

The networks were trained with Adam Optimizer (learning rate $\eta = 1e$-$04$, $\beta_1 = 0.5$, $\beta_2 = 0.9$) for all datasets. The number of discriminator updates was $5$ for each generator update in training. Gradient penalty coefficient for WGAN-GP was set to $10$ for all experiments. The dimension of $z_c$ is the same as the number of classes in the dataset. Most networks used Leaky ReLU activations and Batch Normalization (BN), details for each dataset are provided below. (In the architecture without encoder, Algorithm \ref{alg:bpd} used Adam optimizer to minimize the objective for $5000$ iterations per point.) 

\subsection*{Synthetic Data}

We used batch size = $64$, $z_n$ of $6$ dimensions. LReLU activation with leak = $0.2$ was used. $\beta_n = 10$, $\beta_c = 10$.

\begin{tabular}{|c|c|c|}
\hline
Generator & Encoder & Discriminator \\
\hline
Input $z = (z_n, z_c) \in \mathbb{R}^{10}$  & Input $X \in \mathbb{R}^{16}$   &  Input $X \in \mathbb{R}^{16}$  \\ 
\hline
FC $256$ LReLU BN & FC $256$ LReLU BN &  FC $256$ LReLU BN \\
\hline
FC $256$ LReLU BN & FC $256$ LReLU BN  & FC $256$ LReLU BN \\
\hline
FC $16$ Sigmoid  & FC $10$ linear for $\hat{z}$ & FC $1$ linear  \\
&Softmax on last $4$ to obtain $\hat{z_c}$  &  \\
\hline
\end{tabular}

\subsection*{MNIST and Fashion-MNIST}

We used batch size = $64$, $z_n$ of $30$ dimensions. LReLU activation with leak = $0.2$ was used. $\beta_n = 10$ for MNIST and $\beta_n = 0$ for Fashion-MNIST, $\beta_c = 10$ for both .

\begin{table}[h!]
\centering
\begin{tabular}{|c|c|c|}
\hline
Generator & Encoder & Discriminator \\
\hline
Input $z = (z_n, z_c) \in \mathbb{R}^{40}$  & Input $X \in \mathbb{R}^{28 \times 28}$   &  Input $X \in \mathbb{R}^{28 \times 28}$  \\ 
\hline
FC $1024$ ReLU BN & $4 \times 4$ conv, &  $4 \times 4$ conv, \\
                   & 64 stride 2 LReLU &  64 stride 2 LReLU \\
\hline
FC $7 \times 7 \times 128$ ReLU BN &  $4 \times 4$ conv, & $4 \times 4$ conv, \\
                                   &  128 stride 2 LReLU & 128 stride 2 LReLU \\
\hline
$4 \times 4$ upconv,   & FC $1024$ LReLU & FC $1024$ LReLU  \\
64 stride 2 ReLU BN    &                  &   \\
\hline
$4 \times 4$ upconv,  &FC $40$ linear for $\hat{z}$  &  FC $1$ linear \\
1 stride 2, Sigmoid &Softmax on last $10$ to obtain $\hat{z_c}$  &  \\
\hline
\end{tabular}
\end{table}

For Fashion-MNIST, we used $z_n = 40$. Rest of the architecture remained identical.

\subsection*{10x\_73k}

We used batch size = $64$, $z_n$ of $30$ dimensions. LReLU activation with leak = $0.2$ was used. $\beta_n = 10$, $\beta_c = 10$.

\begin{table}[h!]
\centering
\begin{tabular}{|c|c|c|}
\hline
Generator & Encoder & Discriminator \\
\hline
Input $z = (z_n, z_c) \in \mathbb{R}^{38}$  & Input $X \in \mathbb{R}^{720}$   &  Input $X \in \mathbb{R}^{720}$  \\ 
\hline
FC $256$ LReLU & FC $256$ LReLU  &  FC $256$ LReLU \\
\hline
FC $256$ LReLU & FC $256$ LReLU  & FC $256$ LReLU \\
\hline
FC $720$ Linear  & FC $38$ linear for $\hat{z}$ & FC $1$ linear  \\
&Softmax on last $8$ to obtain $\hat{z_c}$  &  \\
\hline
\end{tabular}
\end{table}

\subsection*{Pendigits}

We used batch size = $64$, $z_n$ of $5$ dimensions. LReLU activation with leak = $0.2$ was used. $\beta_n = 10$, $\beta_c = 10$.

\begin{table}[h!]
\centering
\begin{tabular}{|c|c|c|}
\hline
Generator & Encoder & Discriminator \\
\hline
Input $z = (z_n, z_c) \in \mathbb{R}^{15}$  & Input $X \in \mathbb{R}^{16}$   &  Input $X \in \mathbb{R}^{16}$  \\ 
\hline
FC $256$ LReLU BN & FC $256$ LReLU BN &  FC $256$ LReLU BN \\
\hline
FC $256$ LReLU BN & FC $256$ LReLU BN & FC $256$ LReLU BN\\
\hline
FC $16$ Sigmoid  & FC $15$ linear for $\hat{z}$ & FC $1$ linear  \\
&Softmax on last $10$ to obtain $\hat{z_c}$  &  \\
\hline
\end{tabular}
\end{table}

\subsection*{Coil-20, Coil-100 and CIFAR-10}

For Coil-20, we used batch size = $64$, $z_n$ of $20$ dimensions. For Coil-100, we used batch size = $512$, $z_n$ of $20$ dimensions. For CIFAR-10, we used batch size = $64$, $z_n$ of $50$ dimensions. 

$\beta_n = 10, \beta_c = 10$, LReLU activation with leak = $0.2$ for all datasets.

\begin{table}[h!]
\centering
\begin{tabular}{|c|c|c|}
\hline
Generator & Encoder & Discriminator \\
\hline
Input $z = (z_n, z_c) \in \mathbb{R}^{d_z}$  & Input $X \in \mathbb{R}^{32 \times 32 \times 3}$   &  Input $X \in \mathbb{R}^{32 \times 32 \times 3}$  \\ 
\hline
FC $2 \times 2 \times 448$  & $4 \times 4$ conv, &  $4 \times 4$ conv, \\
       ReLU BN            & 64 stride 2 LReLU    &  64 stride 2 LReLU \\
\hline
$4 \times 4$ upconv,   & $4 \times 4$ conv,       & $4 \times 4$ conv,  \\
256 stride 2 ReLU BN    &  128 stride 2 LReLU BN  &  128 stride 2 LReLU BN\\
\hline
$4 \times 4$ upconv,   & $4 \times 4$ conv,       & $4 \times 4$ conv,  \\
128 stride 2 ReLU BN    &  256 stride 2 LReLU BN  & 256 stride 2 LReLU BN  \\
\hline
$4 \times 4$ upconv,   & $4 \times 4$ conv, & $4 \times 4$ conv,  \\
64 stride 2 ReLU BN    & 512 stride 2 LReLU BN & 512 stride 2 LReLU BN   \\
\hline
$4 \times 4$ upconv,  &FC $d_z$ linear for $\hat{z}$  &  FC $1$ linear \\
3 stride 2, Sigmoid &Softmax on last $K$ to obtain $\hat{z_c}$  &  \\
\hline
\end{tabular}
\end{table}

For InfoGAN, we used the implementation of the authors \url{https://github.com/openai/InfoGAN} for MNIST and Fashion-MNIST. For the other datasets, we used our hyperparameters for Generator and Discriminator and added the Q network (FC 128-BN-LReLU-FC dim $z_c$). For ``GAN with bp'', we used the same Generator and Discriminator hyperparameters as ClusterGAN. Features for ``GAN with Disc. $\phi$'' was obtained from the trained Discriminator of experiments ``GAN with bp''.

\subsection{Reporting Clustering Performance}

In \cite{xie}, the authors ran all algorithms multiple times with a single hyperparamater change and reported the best accuracy. For fair comparison, we used $5$ runs to determine the best model using validation. To be more precise, we first split our datasets into Train, Validation and Test portions. The GAN was trained only on the Train split of the data in an unsupervised manner, sometimes with a single hyper-parameter change such as $\beta_n, z_n$, leak or batch-norm . For each dataset, we saved the model with the best purity on the Validation split from the $5$ runs. Table \ref{table:compare_alg} reports the metrics on the Test split for the saved model. The metrics on the entire dataset for the saved model was either identical upto $2$ decimals or slightly better than Test. But we report only for the Test split, since it has neither been used for training nor for validation runs.

\newpage

\subsection{Generated Modes}

%\subsection{MNIST}

\begin{figure}[h!]
\centering 
\begin{subfigure}[b]{0.30\textwidth}
\includegraphics[width=\textwidth]{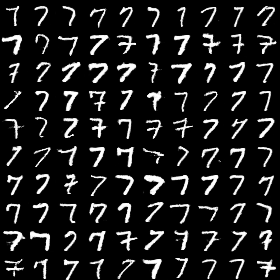} 
\caption{mode $0$}
\end{subfigure}
\begin{subfigure}[b]{0.30\textwidth}
\includegraphics[width=\textwidth]{mnist_mode1_samples.png} 
\caption{mode $1$}
\end{subfigure}
\begin{subfigure}[b]{0.30\textwidth}
\includegraphics[width=\textwidth]{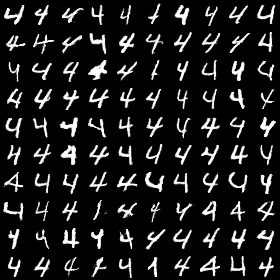} 
\caption{mode $2$}
\end{subfigure}
\begin{subfigure}[b]{0.30\textwidth}
\includegraphics[width=\textwidth]{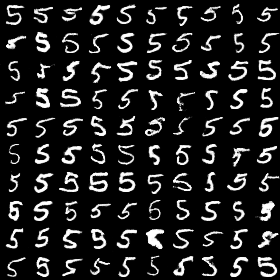} 
\caption{mode $3$}
\end{subfigure}
\begin{subfigure}[b]{0.30\textwidth}
\includegraphics[width=\textwidth]{mnist_mode4_samples.png} 
\caption{mode $4$}
\end{subfigure}
\begin{subfigure}[b]{0.30\textwidth}
\includegraphics[width=\textwidth]{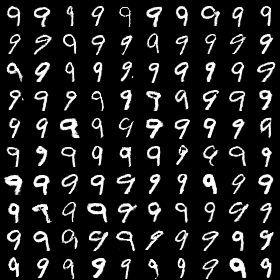} 
\caption{mode $5$}
\end{subfigure}
\begin{subfigure}[b]{0.30\textwidth}
\includegraphics[width=\textwidth]{mnist_mode6_samples.png} 
\caption{mode $6$}
\end{subfigure}
\begin{subfigure}[b]{0.30\textwidth}
\includegraphics[width=\textwidth]{mnist_mode7_samples.png} 
\caption{mode $7$}
\end{subfigure}
\begin{subfigure}[b]{0.30\textwidth}
\includegraphics[width=\textwidth]{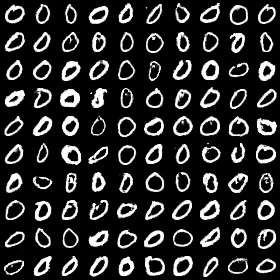} 
\caption{mode $8$}
\end{subfigure}
\begin{subfigure}[b]{0.30\textwidth}
\includegraphics[width=\textwidth]{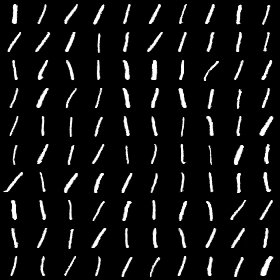} 
\caption{mode $9$}
\end{subfigure}
\caption{Generated digits from distinct modes}
\label{modes_mnist}
\end{figure}

\newpage

%\subsection{Fashion-MNIST}

\begin{figure}[h!]
\centering 
\begin{subfigure}[b]{0.30\textwidth}
\includegraphics[width=\textwidth]{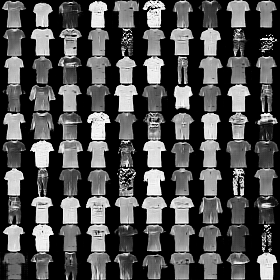} 
\caption{mode $0$}
\end{subfigure}
\begin{subfigure}[b]{0.30\textwidth}
\includegraphics[width=\textwidth]{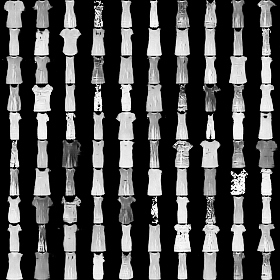} 
\caption{mode $1$}
\end{subfigure}
\begin{subfigure}[b]{0.30\textwidth}
\includegraphics[width=\textwidth]{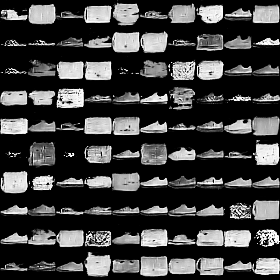} 
\caption{mode $2$}
\end{subfigure}
\begin{subfigure}[b]{0.30\textwidth}
\includegraphics[width=\textwidth]{fashion_mode3_samples.png} 
\caption{mode $3$}
\end{subfigure}
\begin{subfigure}[b]{0.30\textwidth}
\includegraphics[width=\textwidth]{fashion_mode4_samples.png} 
\caption{mode $4$}
\end{subfigure}
\begin{subfigure}[b]{0.30\textwidth}
\includegraphics[width=\textwidth]{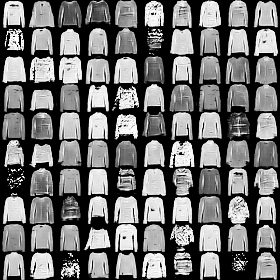} 
\caption{mode $5$}
\end{subfigure}
\begin{subfigure}[b]{0.30\textwidth}
\includegraphics[width=\textwidth]{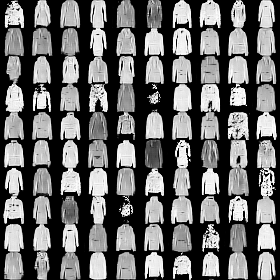} 
\caption{mode $6$}
\end{subfigure}
\begin{subfigure}[b]{0.30\textwidth}
\includegraphics[width=\textwidth]{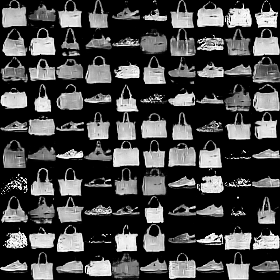} 
\caption{mode $7$}
\end{subfigure}
\begin{subfigure}[b]{0.30\textwidth}
\includegraphics[width=\textwidth]{fashion_mode8_samples.png} 
\caption{mode $8$}
\end{subfigure}
\begin{subfigure}[b]{0.30\textwidth}
\includegraphics[width=\textwidth]{fashion_mode9_samples.png} 
\caption{mode $9$}
\end{subfigure}
\caption{Generated fashion items from distinct modes}
\label{modes_fashion}
\end{figure}

\newpage

%\subsection{CIFAR-10}

\begin{figure}[h!]
\centering 
\begin{subfigure}[b]{0.30\textwidth}
\includegraphics[width=\textwidth]{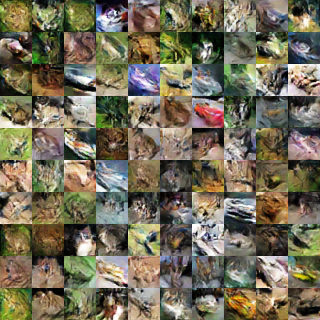} 
\caption{mode $0$}
\end{subfigure}
\begin{subfigure}[b]{0.30\textwidth}
\includegraphics[width=\textwidth]{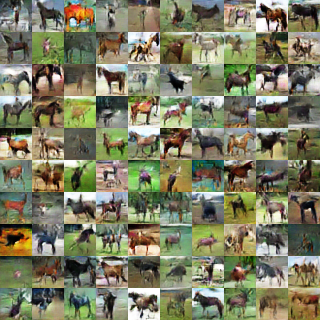} 
\caption{mode $1$}
\end{subfigure}
\begin{subfigure}[b]{0.30\textwidth}
\includegraphics[width=\textwidth]{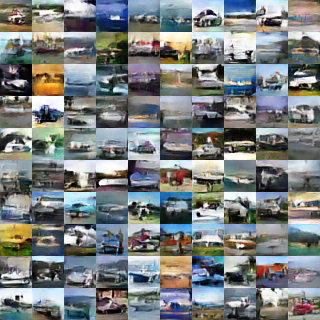} 
\caption{mode $2$}
\end{subfigure}
\begin{subfigure}[b]{0.30\textwidth}
\includegraphics[width=\textwidth]{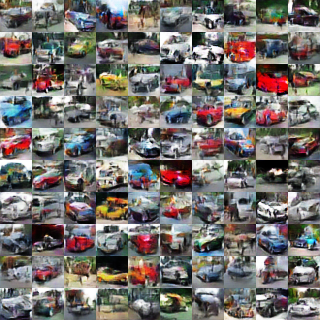} 
\caption{mode $3$}
\end{subfigure}
\begin{subfigure}[b]{0.30\textwidth}
\includegraphics[width=\textwidth]{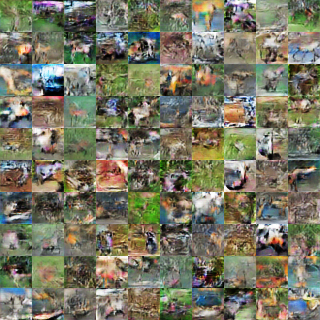} 
\caption{mode $4$}
\end{subfigure}
\begin{subfigure}[b]{0.30\textwidth}
\includegraphics[width=\textwidth]{cifar10_mode5_samples.png} 
\caption{mode $5$}
\end{subfigure}
\begin{subfigure}[b]{0.30\textwidth}
\includegraphics[width=\textwidth]{cifar10_mode6_samples.png} 
\caption{mode $6$}
\end{subfigure}
\begin{subfigure}[b]{0.30\textwidth}
\includegraphics[width=\textwidth]{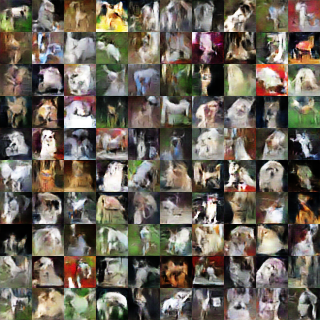} 
\caption{mode $7$}
\end{subfigure}
\begin{subfigure}[b]{0.30\textwidth}
\includegraphics[width=\textwidth]{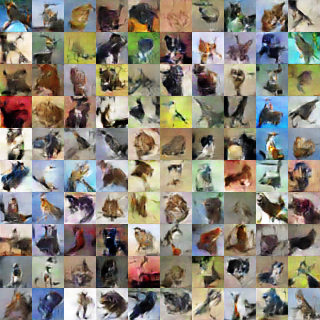} 
\caption{mode $8$}
\end{subfigure}
\begin{subfigure}[b]{0.30\textwidth}
\includegraphics[width=\textwidth]{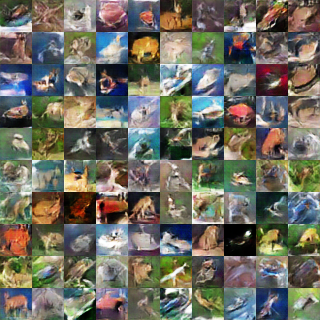} 
\caption{mode $9$}
\end{subfigure}
\caption{Generated categories from distinct modes of CIFAR-10}
\label{modes_fashion}
\end{figure}

\end{document}